\documentclass[11pt]{article}

\usepackage[preprint]{acl}

\usepackage{times}
\usepackage{latexsym}

\usepackage[T1]{fontenc}

\usepackage[utf8]{inputenc}

\usepackage{microtype}

\usepackage{inconsolata}

\usepackage{graphicx}

\usepackage{lipsum}
\usepackage{amsmath}
\usepackage{amsfonts}
\usepackage{multirow}       
\usepackage{subfigure}
\usepackage{xspace}         
\usepackage{enumitem} 
\usepackage[hypcap=true]{caption}
\usepackage{comment}
\usepackage{booktabs}
\usepackage{amsthm}
\usepackage{colortbl}
\usepackage[table]{xcolor}
\usepackage{algorithm}      
\usepackage{algpseudocode}  
\usepackage{tikz}
\usetikzlibrary{arrows.meta, positioning, calc}

\newtheorem{theorem}{Theorem}
\newtheorem{definition}{Definition}
\DeclareMathOperator{\Cov}{Cov}

\newcommand{\mname}{{ReGLU}\xspace}
\newcommand{\initname}{{RILA}\xspace}
\newcommand{\regname}{{ROL}\xspace}

\definecolor{basebg}{HTML}{F5F5F5}     
\definecolor{methodbg}{HTML}{E3F2FD}    
\definecolor{bestcolor}{HTML}{1565C0} 

\newcommand{\basecc}[1]{\cellcolor{basebg}#1}
\newcommand{\methodcc}[1]{\cellcolor{methodbg}#1}

\definecolor{heat1}{HTML}{F7FBFF}
\definecolor{heat2}{HTML}{DEEBF7}
\definecolor{heat3}{HTML}{C6DBEF}
\definecolor{heat4}{HTML}{9ECAE1}
\definecolor{heat5}{HTML}{6BAED6}
\definecolor{heat6}{HTML}{4292C6}

\newcommand{\heatB}[1]{\cellcolor{heat2}#1}
\newcommand{\heatC}[1]{\cellcolor{heat3}#1}
\newcommand{\heatD}[1]{\cellcolor{heat4}#1}
\newcommand{\heatE}[1]{\cellcolor{heat5}#1}
\newcommand{\heatF}[1]{\cellcolor{heat6}#1}

\title{Representation-Guided Parameter-Efficient LLM Unlearning}

\author{%
    Zeguan Xiao$^{1*}$, \
    Lang Mo$^{2}$\thanks{\ \ Equal Contribution.}, \
    Yun Chen$^{1}$,
    Lei Yang$^{3}$,
    Jiehui Zhao$^{3}$\\
    \textbf{Lili Yang}$^{2}$\thanks{Corresponding Authors.},
    \textbf{Guanhua Chen}$^{2}$\footnotemark[2]  \ \\
    $^1$Shanghai University of Finance and Economics  \\
    $^2$Southern University of Science and Technology, $^3$Deepexi Technology Co. Ltd. 
}

\begin{document}
\maketitle
\begin{abstract}
Large Language Models (LLMs) often memorize sensitive or harmful information, necessitating effective machine unlearning techniques. While existing parameter-efficient unlearning methods have shown promise, they still struggle with the forget-retain trade-off.
This can be attributed to their reliance on parameter importance metrics to identify parameters that are important exclusively for forget set, which is fundamentally limited by the superposition phenomenon. Due to the polysemantic nature of LLMs parameters, such an importance metric may struggle to disentangle parameters associated with forget and retain sets.
In this work, we propose \textbf{Re}presentation-\textbf{G}uided \textbf{L}ow-rank \textbf{U}nlearning (\mname), a novel approach that leverages the geometric properties of representation spaces to achieve robust and precise unlearning. First, we develop a representation-guided initialization for LoRA that identifies the optimal subspace for selective forgetting. Second, we introduce a regularization loss that constrains the outputs of the LoRA update to lie in the orthogonal complement of the retain set's representation subspace, thereby minimizing interference with the model's performance on the retain set. We evaluate \mname on the TOFU and WMDP benchmarks across multiple models. Our results demonstrate that \mname consistently outperforms state-of-the-art baselines, achieving superior unlearning quality while maintaining higher model utility.
\end{abstract}

\section{Introduction}

Large Language Models (LLMs) have achieved remarkable success across various natural language processing tasks, demonstrating emergent abilities in reasoning, coding, and creative writing \citep{wei2022emergent}. However, these models are often trained on massive, uncurated datasets that may contain sensitive personal information, copyrighted content, or harmful biases \citep{brown2022privacy, bender2021dangers}. The tendency of LLMs to memorize and regenerate such data poses significant risks to privacy and intellectual property \citep{carlini2022quantifying, nasr2023scalable}. Consequently, there is an urgent need for Machine Unlearning (MU) \citep{cao2015towards}, which aims to remove specific knowledge from a pre-trained model without the prohibitive cost of retraining from scratch.

Most existing LLM unlearning methods rely on full fine-tuning, i.e., updating all model parameters. However, modifying billions of parameters is computationally expensive and, more critically, substantially increases the risk of catastrophic forgetting. To address this issue, recent studies have leveraged LoRA \citep{hu2022lora} for LLM unlearning and demonstrated that unlearning performance can be comparable to, or even better than, full fine-tuning while substantially reducing computational cost \citep{cha2025fila,kim2025improving}.

Despite this progress, these methods still struggle with the forget-retain trade-off: reducing performance on the forget set often comes at the cost of degraded performance on the retain set \citep{fisher1922mathematical, cha2025fila, kim2025improving, xiao2026modeling}. We hypothesize that this limitation stems from their reliance on estimating parameter importance (e.g., via Fisher information \citep{fisher1922mathematical}) while neglecting the phenomenon of superposition \citep{elhage2022toy} in LLMs. The superposition phenomenon implies that a single parameter is often involved in the representation of multiple concepts, leading to polysemanticity of parameters. Consequently, relying on such importance measures to isolate forget-related parameters becomes problematic, as these parameters are often simultaneously crucial for maintaining performance on the retain set or other unseen general knowledge.

Our approach is built on a key insight: while parameter importance measures may be unreliable due to superposition, the representation subspaces can be more effectively disentangled \citep{zou2023representation}.
By constraining unlearning updates within a subspace that aligns with forget-set representations while minimizing interference with retain-set representations, we can more effectively isolate forget-related knowledge and preserve model utility.

In this work, we propose \textbf{Re}presentation-\textbf{G}uided \textbf{L}ow-rank \textbf{U}nlearning (\mname), a novel approach that leverages the geometric properties of representation spaces to achieve robust and precise unlearning. Our approach consists of two components.
First, we develop \initname (\textbf{R}epresentation-guided \textbf{I}nitialization of \textbf{L}ow-rank \textbf{A}daption), a LoRA initialization strategy for LLM unlearning. \initname identifies a balanced subspace that maximizes the variance of forget-set representations while minimizing that of the retain set.
Second, we introduce \regname (\textbf{R}epresentation \textbf{O}rthogonal \textbf{L}oss), a regularization loss term for unlearning. By identifying the principal subspace of the retain set's representations, this loss enforces an orthogonality constraint on the LoRA up-projection matrices. This ensures that the LoRA update lies in the orthogonal complement of a subspace of the representations of retain set, thereby minimizing the interference with the original model's behavior.

We evaluate \mname on two widely used LLM unlearning benchmarks: TOFU \citep{maini2024tofu} and WMDP \citep{li2024wmdp}. Our results demonstrate that \mname consistently outperforms baselines.

Our contributions are summarized as follows\footnote{Code: \url{https://github.com/sustech-nlp/ReGLU}}:
\begin{itemize}[leftmargin=*]
    \item \textbf{Methodology:} We propose \mname, a novel framework that shifts the LoRA-based LLM unlearning paradigm from parameter importance to representation geometry.
    \item \textbf{Experiments:} We conduct extensive evaluations on the TOFU and WMDP benchmarks across multiple models, including Llama-2-7B \citep{touvron2023llama}, Phi-1.5 \citep{li2023phi15}, and Zephyr-7B-beta \citep{tunstall2024zephyr}. The results demonstrate that \mname consistently establishes new state-of-the-art performance.
    \item \textbf{Analysis:} We provide a theoretical guarantee for our initialization strategy and offer in-depth geometric diagnostics. Our analysis confirms that \mname successfully disentangles forget and retain representations, validating that subspace-level control is more robust than parameter-level estimation for unlearning.
\end{itemize}

\begin{figure*}[t]
    \centering
    \includegraphics[width=\textwidth]{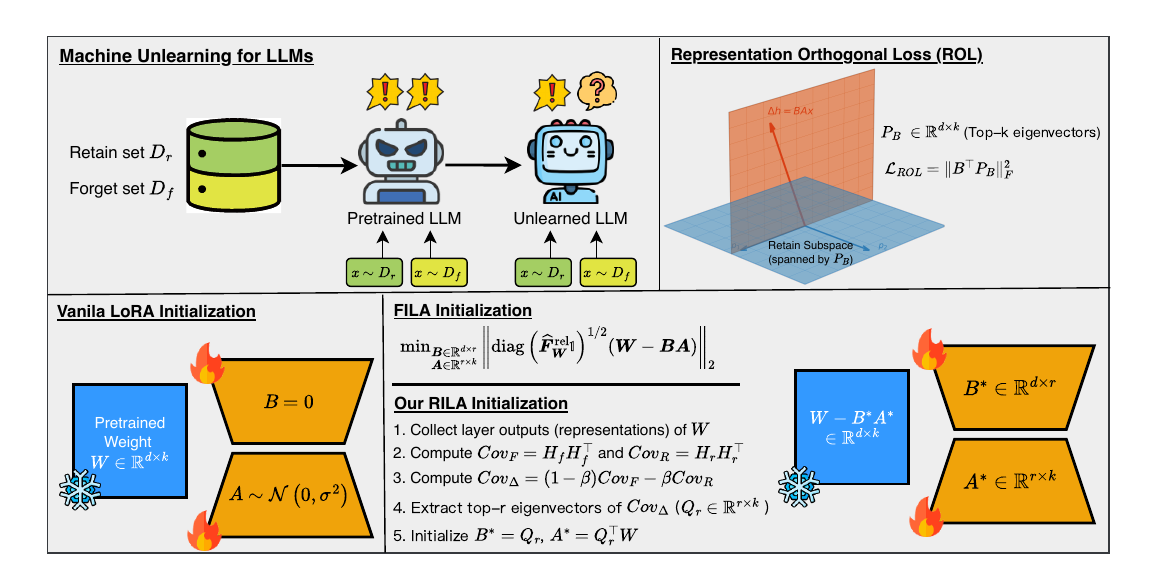}
    \caption{LLM unlearning aims to remove specific information from a pre-trained model. FILA estimates parameter importances $\widehat{\boldsymbol{F}}_{\boldsymbol{W}}^{\mathrm{rel}}$ and solves a weighted low-rank approximation problem to initialize LoRA matrices. Our \mname framework collects layer representations and leverages representation geometry to guide selective forgetting while preserving retain set knowledge.}
    \label{fig:main_architecture}
\end{figure*}

\section{Background}

\subsection{Problem Definition}
For a model $\mathcal{M}$ that is trained on a dataset $\mathcal{D}$, Machine Unlearning (MU) \citep{cao2015towards} aims to remove specific information from $\mathcal{M}$, resulting in an unlearned model $\mathcal{M}'$ that no longer retains or utilizes this undesired information.
Formally, we define the information to forget as a subset of $\mathcal{D}$, called the \textit{forget set} $\mathcal{D}_f$.
Ideally, after unlearning, the model should behave as if only trained on \( \mathcal{D}_r = \mathcal{D} \setminus \mathcal{D}_f\), referred to as the \textit{retain set}.

In the context of LLM unlearning, the forget set $\mathcal{D}_f$ and retain set $\mathcal{D}_r$ are text corpora. The unlearning process typically involves fine-tuning the original model $\mathcal{M}$ on $\mathcal{D}_f$, optionally using the retain set $\mathcal{D}_r$, with specific objectives to obtain $\mathcal{M}'$.

\subsection{Low-Rank Adaptation (LoRA)}
LoRA \citep{hu2022lora} is a parameter-efficient fine-tuning technique that allows effective adaptation with significantly fewer trainable parameters.
Specifically, for a linear layer with weight matrix \(W_0 \in \mathbb{R}^{d_{out} \times d_{in}}\), LoRA parameterizes the update as a low-rank update $\Delta W$:
\begin{equation*}
W = W_0 + \Delta W = W_0 + BA,
\end{equation*}
where \(B \in \mathbb{R}^{d_{out} \times r}\) and \(A \in \mathbb{R}^{r \times d_{in}}\) are the low-rank matrices with \(r \ll \min(d_{out}, d_{in})\).
During finetuning, only the matrices \(B\) and \(A\) are updated, while the original weights \(W_0\) remain fixed.

\subsection{LLM Unlearning Methods}
\label{sec:llm_unlearning_methods}

Given a pre-trained LLM with parameters \(\boldsymbol{\theta}\), we denote the probability distribution it defined as \(p(x; \boldsymbol{\theta})\), where \(x\) represents the input text.

The most prevalent approach to unlearning is to suppress the model's likelihood on the forget set \(\mathcal{D}_f\)—that is, drive $p(x;\boldsymbol{\theta})$ downward for $x\in\mathcal{D}_f$.
This is commonly implemented by performing Gradient Ascent (GA) on the cross-entropy objective (equivalently, minimizing the negative cross-entropy) over $\mathcal{D}_f$:
\begin{equation*}
\mathcal{L}_{\mathrm{GA}}(\mathcal{D}_f; \boldsymbol{\theta})
= - \mathbb{E}_{x \sim \mathcal{D}_f}\big[-\log p(x; \boldsymbol{\theta})\big].
\end{equation*}
Minimizing the above GA objective reduces the assigned probabilities $p(x;\boldsymbol{\theta})$, achieving the goal of minimizing the forget-set likelihood.

Given the inherent issues of GA \citep{zhang2024negative,cha2025fila}, some methods have been proposed as alternative unlearning objectives. For instance, NPO \citep{zhang2024negative} and SimNPO \citep{fan2024simplicity} regularize GA by transforming the unbounded objective into a bounded one and applying adaptive smoothing to the forget-set gradients, allowing for more controlled divergence during unlearning, preventing catastrophic collapse.

Inverted Hinge Loss (IHL) \citep{cha2025fila} is another approach designed to simultaneously address the issues of unbounded loss, gradient spread, and the degradation of generative performance:
\begin{equation*}
\mathcal{L}_{\mathrm{IHL}}(\boldsymbol{x})
= 1 + p\left(x_t | x_{<t}; \boldsymbol{\theta}\right)
- \max_{v \neq x_t} p\left(v | x_{<t}; \boldsymbol{\theta}\right).
\end{equation*}

FILA \citep{cha2025fila} and VILA \citep{kim2025improving} employ Fisher Information (FI) to identify parameters associated with a dataset.
The FI of a dataset $\mathcal{D}$ with respect to model parameters $\boldsymbol{\theta}$ is defined as:
\begin{equation*}
\mathcal{F}_{\boldsymbol{\theta}}(\mathcal{D}) = \mathbb{E}_{\mathcal{D}} \left[ \left( \frac{\partial}{\partial \boldsymbol{\theta}} \log p(x; \boldsymbol{\theta}) \right)^2 \right].
\end{equation*}
FILA computes the ratio of FI values from the forget set and retain set to determine the \textit{forget importance map} $\mathcal{S} = \mathcal{F}_{\boldsymbol{\theta}}(\mathcal{D}_f) / \mathcal{F}_{\boldsymbol{\theta}}(\mathcal{D}_r)$.
By leveraging this map, FILA initializes the LoRA matrices $A$ and $B$ such that $BA$ captures the components of the original weight matrix $W_0$ most relevant to the forget set.
VILA extends FILA by addressing the inaccuracy in FI as a variance measure when the score function has a non-zero expectation.

\section{Methodology}
\label{sec:methodology}

In this section, we present \mname, a representation-guided framework for LoRA-based LLM unlearning. Our framework consists of two complementary components. First, in Section~\ref{sec:initialization}, we introduce \initname, a LoRA initialization strategy that leverages the geometric structure of representation spaces to align the initial LoRA update with a subspace maximally discriminative between the forget and retain sets. Second, in Section~\ref{sec:regularization}, we propose \regname, a regularization loss that enforces orthogonality between the LoRA update and the principal subspace of retain-set representations, preventing interference with preserved knowledge during training. We conclude with the overall optimization objective that combines these components.

\subsection{\initname: A Representation-Guided LoRA Initialization for LLM Unlearning}
\label{sec:initialization}

\subsubsection{Motivation}
Consider a linear layer $h = W_0 x$, where $x \in \mathbb{R}^{d_{in}}$ is the input representation, $W_0 \in \mathbb{R}^{d_{out} \times d_{in}}$ is the pre-trained weight matrix, and $h \in \mathbb{R}^{d_{out}}$ is the output representation. 
Fine-tuning for unlearning aims to learn a parameter update $\Delta W$ such that the updated weight $W = W_0 + \Delta W$ reduces the likelihood of the forget set $\mathcal{D}_f$ while maintaining performance on the retain set $\mathcal{D}_r$.
Ideally, the update $\Delta W$ should have minimal impact on the outputs for inputs from the retain set. This can be expressed as $(W_0 + \Delta W) x \approx W_0 x$, which implies $\Delta W x \approx 0$ for all $x \in \mathcal{D}_r$. For the forget set $\mathcal{D}_f$, the update $\Delta W$ should produce a substantial change in the output to effectively suppress the model's ability to recall target knowledge.

To translate these requirements into a formal objective, we define the unlearning problem as finding an update $\Delta W$ that maximizes the "differential energy" between the forget and retain sets. Specifically, we aim to maximize the expected change in output norm for the forget set while simultaneously minimizing it for the retain set:
\begin{equation}
\begin{split}
\label{eq:norm_objective}
\max_{\Delta W} \quad & (1 - \beta)\, \mathbb{E}_{x \sim \mathcal{D}_f}\!\left[\|\Delta W x\|_2^2\right] \\
&\quad - \, \beta\, \mathbb{E}_{x \sim \mathcal{D}_r}\!\left[\|\Delta W x\|_2^2\right],
\end{split}
\end{equation}
where $\beta \in [0, 1]$ is a hyperparameter that balances the trade-off between forget and retain.

LoRA-based LLM unlearning methods, FILA \citep{cha2025fila} and VILA \citep{kim2025improving}, have achieved promising results by leveraging informed initialization strategies. Meanwhile, recent research on LoRA \citep{meng2024pissa, wang-etal-2025-milora} highlights that the initialization strategy plays a crucial role in determining the convergence behavior and final performance of the model. Inspired by these successes, we directly model the weight update $\Delta W$ as LoRA and develop our LoRA initialization strategy to solve the unlearning problem by initializing the LoRA matrices to explicitly align with the differential energy objective defined in Eq.~\ref{eq:norm_objective}.

\subsubsection{Representation-Guided Initialization}

Following the motivation in the previous section, we aim to initialize the LoRA matrices $A$ and $B$ such that the update $\Delta W = BA$ is initially aligned with a subspace that maximizes the impact on the forget set $\mathcal{D}_f$ while minimizing the interference with the retain set $\mathcal{D}_r$.

Our key insight is to leverage the decomposability of representations \citep{zou2023representation} to guide the initialization of LoRA. This approach circumvents the issues of parameter polysemanticity caused by the superposition phenomenon \citep{elhage2022toy}, which can be a potential limitation when directly employing parameter importance to initialize LoRA matrices.

We formalize our initialization strategy and its theoretical justification in the following theorem:
\begin{theorem}
\label{theorem:init-of-BA}
Consider the distributions $\mathcal{P}_{F}$ and $\mathcal{P}_{R}$ of the output representation $h = W_0 x$ for inputs $x$ sampled from $\mathcal{D}_f$ and $\mathcal{D}_r$, respectively. Let $\Cov_F$ and $\Cov_R$ be their corresponding covariance matrices:
\begin{equation*}
    \Cov_F = \mathbb{E}_{h \sim \mathcal{P}_{F}} [h h^\top], \quad \Cov_R = \mathbb{E}_{h \sim \mathcal{P}_{R}} [h h^\top].
\end{equation*}
Define $\Cov_{\Delta}$ as:
\begin{equation*}
    \Cov_{\Delta} = (1 - \beta) \Cov_F - \beta \Cov_R,
\end{equation*}
where $\beta \in [0, 1]$ is a hyperparameter.
When $A, B$ are initialized as $B_{\text{init}} = Q_r$ and $A_{\text{init}} = Q_r^\top W_0$, where $Q_r = [q_1, q_2, \dots, q_r]$ is the matrix of the top-$r$ eigenvectors of $\Cov_{\Delta}$, Eq.~\ref{eq:norm_objective} is maximized at initialization.
\end{theorem}

The proof of Theorem \ref{theorem:init-of-BA} is provided in Appendix~\ref{app:proof}. Intuitively, this theorem reveals that by constructing a balanced covariance matrix $\Cov_{\Delta}$ that contrasts the forget and retain sets, we can extract eigenvectors that capture the most discriminative directions. These eigenvectors define a subspace where the forget-set representations have high variance while the retain-set representations have low variance, naturally aligning with our goal of selective forgetting.

In practice, since the true distributions $\mathcal{P}_{F}$ and $\mathcal{P}_{R}$ are unknown, we estimate $\Cov_F$ and $\Cov_R$ empirically: feed a small number of samples from $\mathcal{D}_f$ and $\mathcal{D}_r$ through the pre-trained model, collect layer outputs $H_f \in \mathbb{R}^{N_f \times d}$ and $H_r \in \mathbb{R}^{N_r \times d}$, and compute the sample covariances. Using these, build $\Cov_{\Delta}$ and extract $Q_r$ to perform the initialization above. Appendix~\ref{app:cov-estimation} presents a concentration bound for empirical covariance estimation, showing that under bounded activations the empirical estimators converge in spectral norm at the rate $O(M^2\sqrt{\log(d/\delta)/N})$, where $M$ is an upper bound on the activation norm, $d$ is the representation dimension, $N$ is the number of samples, and $\delta$ controls the confidence level of the bound. As a result, the empirical balanced covariance remains close to $\Cov_\Delta$, and its top-$r$ eigenspace is stable whenever the eigengap of $\Cov_\Delta$ is sufficiently large.

\subsection{Representation Orthogonal Loss: A Subspace-Controlled Regularization}
\label{sec:regularization}

To further ensure that the unlearning process maintains performance on the retain set $\mathcal{D}_r$, we introduce Representation Orthogonal Loss (ROL), a subspace-controlled regularization term in training loss. This loss is designed to constrain the output of the LoRA to be orthogonal to the retain set's representations, thereby minimizing interference with the knowledge that should be preserved.

We first identify the principal subspace of the representations for the retain set $\mathcal{D}_r$. Let $H_r = \{h_i\}_{i=1}^N$ be the set of output representations $h = W_0 x$ collected from a small subset of the retain set. We perform eigenvalue decomposition on the covariance matrix of $H_r$ to obtain an orthonormal basis $P_B \in \mathbb{R}^{d_{out} \times k}$, where $k$ is a hyperparameter controls the strength of the regularization. This matrix $P_B$ captures the most critical directions in the representation space that represent the knowledge to be retained.

During training, $P_B$ remains fixed. We define the \regname as:
\begin{equation}
\mathcal{L}_{\regname} = \| \mathbf{B}^\top P_B \|_F^2,
\end{equation}
where $\mathbf{B} \in \mathbb{R}^{d_{out} \times r}$ is the LoRA up-projection matrix and $\|\cdot\|_F$ denotes the Frobenius norm. This loss function regularizes the orthogonality between every pair of columns of $\mathbf{B}$ and $P_B$.
Geometrically, this loss encourages that the LoRA update $\Delta h = B(Ax)$ lies in the orthogonal complement of a subspace of the representations of retain set, thereby minimizing the interference with the original model's behavior on $\mathcal{D}_r$.

\subsection{Overall Algorithm}

The final optimization objective for \mname is a weighted combination of the forget loss, the retain loss, and the orthogonal regularization:
\begin{equation*}
\mathcal{L}_{\text{total}} = \mathcal{L}_{\text{forget}}(\mathcal{D}_f) + \gamma\, \mathcal{L}_{\text{retain}}(\mathcal{D}_r) + \lambda \mathcal{L}_{\regname},
\end{equation*}
where $\mathcal{L}_{\text{forget}}$ can be any of the objectives discussed in Section \ref{sec:llm_unlearning_methods}, and $\gamma, \lambda$ are hyperparameters balancing the different objectives.

The overall unlearning process of \mname is summarized in Algorithm \ref{alg:method} in Appendix~\ref{app:algorithm}. First, we collect a small number of samples from both the forget set $\mathcal{D}_f$ and the retain set $\mathcal{D}_r$ to estimate the covariance matrices $\Cov_F$ and $\Cov_R$. We then compute the balanced covariance matrix $\Cov_\Delta$ and extract its top-$r$ eigenvectors to initialize the LoRA matrices $A$ and $B$ as described in Section 4.2. Simultaneously, we use $\Cov_R$ to construct the orthogonal basis $P_B$. Finally, we optimize the LoRA parameters by minimizing the total loss $\mathcal{L}_{\text{total}}$, which balances unlearning effectiveness, retain set performance, and subspace-controlled regularization.

\section{Experiments}

\subsection{Experimental Setup}

\paragraph{Benchmarks.}
We conduct experiments on two widely used LLM unlearning benchmarks: TOFU \citep{maini2024tofu} and WMDP \citep{li2024wmdp}. TOFU offers 200 diverse synthetic author profiles, each consisting of 20 question-answer pairs. Subsets of these profiles (1\%, 5\%, or 10\%) serve as the forget set for unlearning.
WMDP contains expert-written multiple-choice questions in biosecurity, cybersecurity, and chemistry domains. Following \citet{li2024wmdp}, we use the provided forget corpus and use Wikitext \citep{merity2016pointer} as the retain set. We focus on biosecurity and cybersecurity domains, since the forget corpus for the chemistry domain is not publicly available.

\paragraph{Baselines.}
We compare our method with other LoRA-based unlearning methods, specifically FILA \citep{cha2025fila} and VILA \citep{kim2025improving}, with two unlearning loss functions: GD \citep{liu2022continual} and IHL \citep{cha2025fila}. We focus our evaluation on LoRA-based approaches to ensure a fair comparison within the parameter-efficient framework. This choice is motivated by findings in \citet{cha2025fila}, which demonstrate that LoRA-based unlearning can achieve performance on par with, or even superior to, full fine-tuning.

\paragraph{Models.}
We follow the default configurations for each benchmark. For TOFU, we evaluate unlearning on Llama-2-7B \citep{touvron2023llama} and Phi-1.5B \citep{li2023phi15}. For WMDP, we use Zephyr-7B-beta \citep{tunstall2024zephyr}.

\begin{table*}[t]
\centering
\begin{tabular}{c|l|c|c|c|c}
\toprule
\textbf{Model} & \textbf{Method} & \textbf{Forget 1\%} $\uparrow$ & \textbf{Forget 5\%} $\uparrow$ & \textbf{Forget 10\%} $\uparrow$ & \textbf{AVG.} $\uparrow$ \\
\midrule
\multirow{10}{*}{Phi-1.5B} & \basecc{Original Model} & \basecc{-2.5} & \basecc{-11.5} & \basecc{-16.9} & \basecc{-10.3} \\
  \cmidrule{2-6}
 & GD & -2.5 & -11.2 & -17.3 & -10.3 \\
 & GD + FILA & -1.8 & -8.7 & -11.9 & -7.5 \\
 & GD + VILA & -2.5 & -10.9 & -14.4 & -9.3 \\
 & \methodcc{GD + \mname} & \methodcc{-0.8} & \methodcc{-9.9} & \methodcc{-11.2} & \methodcc{-7.3} \\
  \cmidrule{2-6}
 & IHL & -1.3 & -11.5 & -12.4 & -8.4 \\
 & IHL + FILA & -2.5 & -9.3 & -10.3 & -7.4 \\
 & IHL + VILA & -2.9 & -10.2 & -10.2 & -7.8 \\
 & \methodcc{IHL + \mname} & \methodcc{-0.1} & \methodcc{-5.4} & \methodcc{-7.7} & \methodcc{-4.4} \\
\midrule
\multirow{10}{*}{Llama-2-7B} & \basecc{Original} & \basecc{-2.9} & \basecc{-14.0} & \basecc{-16.9} & \basecc{-11.3} \\
  \cmidrule{2-6}
 & GD & -3.3 & -13.2 & -13.8 & -10.1 \\
 & GD + FILA & -1.3 & -11.2 & -16.6 & -9.7 \\
 & GD + VILA & -2.5 & -13.6 & -14.7 & -10.3 \\
 & \methodcc{GD + \mname} & \methodcc{-0.6} & \methodcc{-12.9} & \methodcc{-13.3} & \methodcc{-8.9} \\
  \cmidrule{2-6}
 & IHL & -2.9 & -14.3 & -16.6 & -11.3 \\
 & IHL + FILA & -2.2 & -11.2 & -7.5 & -6.9 \\
 & IHL + VILA & -2.5 & -8.2 & -11.1 & -7.2 \\
 & \methodcc{IHL + \mname} & \methodcc{-0.7} & \methodcc{-5.4} & \methodcc{-1.1} & \methodcc{-2.4} \\
\bottomrule
\end{tabular}
\caption{Main results on the TOFU benchmark. We report the log-scaled Forget Quality (FQ) for different forget set sizes (1\%, 5\%, and 10\%) on Phi-1.5B and Llama-2-7B. Higher FQ indicates better unlearning performance, signifying that the unlearned model's behavior is statistically closer to a model retrained from scratch. All methods maintain at least 95\% of the original model's utility.}
\label{tab:tofu_results}
\end{table*}

\begin{table}[t]
\centering
\begin{tabular}{l|c|c|c}
\toprule
\textbf{Method} & \textbf{Bio} $\downarrow$ & \textbf{Cyber} $\downarrow$ & \textbf{AVG.} $\downarrow$ \\
\midrule
\rowcolor{basebg} Original Model & 64.7 & 44.5 & 54.6 \\
\midrule
GD & 58.0 & 37.3 & 47.7 \\
GD + FILA & 59.1 & 41.3 & 50.2 \\
GD + VILA & 49.3 & 26.9 & 38.1 \\
\rowcolor{methodbg} GD + \mname & 41.8 & 28.4 & \textbf{35.1} \\
\midrule
IHL & 54.9 & 39.5 & 47.2 \\
IHL + FILA & 60.6 & 37.0 & 48.8 \\
IHL + VILA & 63.6 & 38.4 & 51.0 \\
\rowcolor{methodbg} IHL + \mname & 49.7 & 33.1 & \textbf{41.4} \\
\bottomrule
\end{tabular}
\caption{Main results on the WMDP benchmark. We report the accuracy (\%) on WMDP-Bio and WMDP-Cyber. Lower accuracy indicates better unlearning performance. All reported checkpoints maintain at least 95\% of the original model's MMLU utility.}
\label{tab:wmdp_results}
\end{table}

\paragraph{Evaluation Metrics and Setting.}
We follow the standard evaluation protocols for each benchmark. For TOFU, we employ the Forget Quality (FQ) \citep{maini2024tofu} to measure the extent of data removal, which is the statistical similarity between the unlearned model and an oracle retrained model. Note that FQ values reported in the results are log-scaled to facilitate visualization and comparison across different magnitudes. The model utility is the harmonic mean of nine metrics \citep{maini2024tofu}.
For WMDP, we report the accuracy of WMDP-Bio and WMDP-Cyber to assess the efficacy of unlearning and evaluate model utility using MMLU accuracy \citep{hendrycks2021mmlu}.
To ensure a fair comparison, we search the hyperparameters of all methods and only consider checkpoints that maintain at least 95\% of the original model's utility and select the one achieving the best unlearning performance. Since all methods demonstrate comparable utility under this criterion, we focus on reporting the forgetting performance in the subsequent results.

\subsection{Main Results}

We present the main unlearning results on the TOFU and WMDP benchmarks in Table~\ref{tab:tofu_results} and Table~\ref{tab:wmdp_results}, respectively.

\paragraph{Results on TOFU.}
As shown in Table~\ref{tab:tofu_results}, \mname consistently outperforms the baseline methods across different forget set sizes (1\%, 5\%, and 10\%) and model architectures (Phi-1.5B and Llama-2-7B). 
When using GD as the loss function, \mname achieves an average FQ of -7.3 on Phi-1.5B and -8.9 on Llama-2-7B, surpassing both FILA and VILA. 
The performance gain is even more pronounced when combined with the IHL. Specifically, IHL + \mname achieves the best overall performance, with an average FQ of -4.4 on Phi-1.5B and -2.4 on Llama-2-7B. 
This demonstrates that our \mname effectively complements different unlearning functions.

\paragraph{Results on WMDP.}
The results on the WMDP benchmark, summarized in Table~\ref{tab:wmdp_results}, further validate the effectiveness of \mname. 
GD + \mname achieves the lowest average accuracy of 35.1\% across the Bio and Cyber domains, which is a significant reduction from the original model's 54.6\% and outperforms GD + VILA (38.1\%). 
Similarly, IHL + \mname (41.4\%) shows a substantial improvement over IHL + FILA (48.8\%) and IHL + VILA (51.0\%). 
These results indicate that \mname can more aggressively suppress the forget-set knowledge without compromising the model's general capabilities.

\subsection{Ablation Studies}

\begin{table}[t]
\centering
\begin{tabular}{l|c|c|c}
\toprule
\textbf{Method} & \textbf{Bio} $\downarrow$ & \textbf{Cyber} $\downarrow$ & \textbf{AVG.} $\downarrow$ \\
\midrule
GD & 58.0 & 37.3 & 47.7 \\
GD + \initname & 54.0 & 25.8 & 39.9 \\
GD + \regname & 44.6 & 36.1 & 40.4 \\
\rowcolor{methodbg} GD + \mname & 41.8 & 28.4 & \textbf{35.1} \\
\midrule
IHL & 54.9 & 39.5 & 47.2 \\
IHL + \initname & 51.2 & 36.7 & 44.0 \\
IHL + \regname & 54.0 & 36.1 & 45.0 \\
\rowcolor{methodbg} IHL + \mname & 49.7 & 33.1 & \textbf{41.4} \\
\bottomrule
\end{tabular}
\caption{Ablation results on the WMDP benchmark. We compare the full \mname framework with two variants: \initname (representation-guided initialization only) and \regname (subspace regularization only). Lower values indicate better unlearning performance.}
\label{tab:wmdp_ablation}
\end{table}

To investigate the contribution of each component in \mname, we conduct ablation studies on the WMDP benchmark using both GD and IHL. We compare the full \mname framework with two variants: (1) \initname, which only employs the representation-guided initialization described in Section \ref{sec:initialization} without the subspace regularization loss \regname; and (2) \regname, which only applies the subspace regularization loss described in Section \ref{sec:regularization} while using standard LoRA initialization.
As shown in Table \ref{tab:wmdp_ablation}, both the initialization and the regularization loss contribute to the overall performance.
Specifically, \initname achieves better unlearning performance than the base GD and IHL, demonstrating that aligning the LoRA update with the forget-set-dominant subspace provides a strong starting point for unlearning. Similarly, \regname also shows improvements over the base methods, indicating that the subspace regularization loss effectively guides the optimization process. Most importantly, the full \mname framework, which combines both components, achieves the best results. This suggests a strong synergy between the geometric initialization and the structural regularization, where the initialization provides a well-aligned starting subspace and the regularization ensures that the subsequent updates remain within a safe region that minimizes interference with the retain set.

\subsection{Analysis of Hyperparameters}

In this section, we investigate the sensitivity of \mname to two key hyperparameters: the LoRA rank $r$ and the balance parameter $\beta$ in Eq.~\ref{eq:norm_objective}. We conduct experiments on TOFU-10\% and WMDP-Cyber.

\paragraph{Impact of LoRA Rank $r$.}
The rank $r$ determines the dimensionality of the low-rank update $\Delta W$. As shown in Table \ref{tab:rank}, we evaluate the performance across different ranks. A higher rank provides more degrees of freedom to align the update with the forget-set subspace, potentially leading to more effective unlearning. However, excessively large ranks may increase the risk of interfering with the retain-set knowledge if the orthogonal regularization is not sufficiently strong. Our results indicate that a rank of $r=16$ or $r=32$ typically strikes a good balance between unlearning efficacy and utility preservation.

\begin{table}[t]
\centering
\begin{tabular}{l|c|c}
\toprule
\textbf{Rank $r$} & \textbf{Forget 10\%} $\uparrow$ & \textbf{WMDP Cyber} $\downarrow$ \\
\midrule
8   & -3.3 & 30.5 \\
16  & -2.8 & 27.9 \\
32  & -1.4 & 28.4 \\
\bottomrule
\end{tabular}
\caption{Impact of LoRA rank $r$ on TOFU and WMDP benchmarks.}
\label{tab:rank}
\end{table}

\paragraph{Impact of Balance Parameter $\beta$.}
The parameter $\beta \in [0, 1]$ in Eq.~\ref{eq:norm_objective} regulates the trade-off between maximizing forget-set variance and minimizing retain-set interference during subspace identification. We vary $\beta$ while keeping other settings fixed, and report the resulting FQ on TOFU-10\% and unlearning performance on WMDP-Cyber.
Table~\ref{tab:beta} shows a clear ``sweet spot'' at $\beta=0.3$, which achieves the best trade-off across both benchmarks (highest FQ on TOFU-10\% and lowest accuracy on WMDP-Cyber). When $\beta$ is too small (e.g., $0.1$), the objective over-emphasizes maximizing forget-set variance, yielding a less stable subspace for forgetting. Conversely, as $\beta$ increases (e.g., $\beta\ge 0.7$), the balanced covariance $\Cov_{\Delta}=(1-\beta)\Cov_F-\beta\Cov_R$ becomes dominated by the retain term, leading to a subspace that is overly conservative and thus weakens forgetting.

\begin{table}[t]
\centering
\begin{tabular}{l|c|c}
\toprule
\boldmath$\beta$ & \textbf{Forget 10\%} $\uparrow$ & \textbf{WMDP Cyber} $\downarrow$ \\
\midrule
0.1   & \heatE{-2.2} & \heatD{38.2} \\
0.3   & \heatF{-2.1} & \heatF{30.1} \\
0.5   & \heatD{-2.3} & \heatE{36.6} \\
0.7   & \heatC{-3.5} & \heatB{42.6} \\
0.9   & \heatB{-3.6} & \heatC{40.6} \\
\bottomrule
\end{tabular}
\caption{Impact of the balance parameter $\beta$ in Eq.~\ref{eq:norm_objective}. Darker colors indicate better performance.}
\label{tab:beta}
\end{table}

\begin{figure}[t]
\centering
\includegraphics[width=0.95\columnwidth]{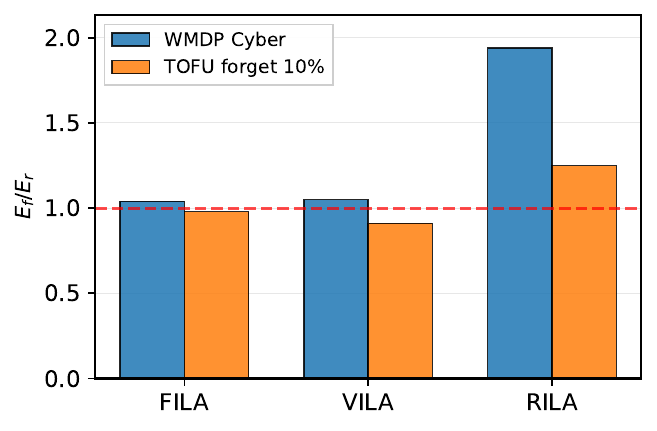}
\caption{Comparison of activation norms at initialization. \initname (our proposed initialization) achieves a significantly higher forget-to-retain energy ratio compared to FILA and VILA.}
\label{fig:init_energy}
\end{figure}

\begin{figure}[t]
\centering
\includegraphics[width=0.95\columnwidth]{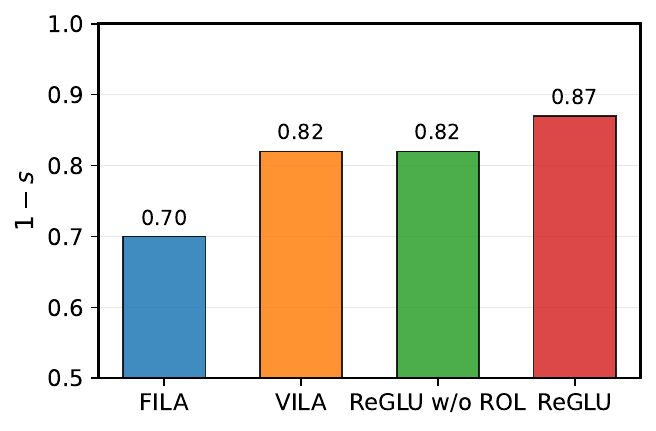}
\caption{Orthogonality analysis between LoRA $B$ matrix and retain subspace $P_B$. Higher values of $1 - s$ indicate greater orthogonality to the retain representation subspace, which is desirable for effective unlearning while preserving retain-set knowledge.}
\label{fig:b_subspace}
\end{figure}

\subsection{Analysis of Mechanism}
\label{sec:mechanism}

To understand the underlying mechanism of \mname, we analyze how the LoRA update interacts with the forget and retain sets at initialization and during training.

Figure~\ref{fig:init_energy} shows that \initname achieves a significantly higher forget-to-retain energy ratio at initialization compared to FILA and VILA, indicating that \initname better aligns the LoRA update with the forget-set subspace while minimizing interference with the retain set.

To analyze the subspace of LoRA outputs across different methods, we examine the angular distance between the columns of the LoRA $B$ matrix and the columns of the $P_B$ matrix (detailed in Section~\ref{sec:regularization}), which measures the influence of LoRA outputs on the retain representation subspace. Specifically, we compute the average pairwise cosine similarity between columns of $B$ and $P_B$, denoted as $s = \text{avg}_{i,j} \cos^2(B[:,i], P_B[:,j])$, and report $1 - s$, which eliminates the effect of the $B$ matrix's scale. Higher values indicate greater orthogonality to the retain subspace, which is desirable for effective unlearning.
We conduct experiments on Llama-2-7B with the Forget 5\% setting.
As shown in Figure~\ref{fig:b_subspace}, \mname achieves the highest orthogonality score (0.87), substantially outperforming both FILA (0.70) and VILA (0.82). This demonstrates that our method more effectively maintains LoRA updates orthogonal to the retain representation subspace. Furthermore, removing the \regname component (ReGLU - ROL, 0.82) results in a noticeable decrease in orthogonality, confirming the effectiveness of our subspace-controlled regularization in constraining the LoRA outputs to remain orthogonal to the retain set's principal directions.

\begin{table*}[t]
\centering
\begin{tabular}{l|c|c|c|c|c}
\toprule
\textbf{Method} & \textbf{Params} & \textbf{Rand.\ SVD} & \textbf{Forget 1\%} & \textbf{Forget 5\%} & \textbf{Forget 10\%} \\
\midrule
FILA & 7B & -- & 0.71 & 3.13 & 23.32 \\
VILA & 7B & -- & 0.09 & 0.38 & 0.81 \\
\rowcolor{methodbg} \initname & 7B & $\times$ & 0.50 & 0.50 & 0.51 \\
\rowcolor{methodbg} \initname & 7B & \checkmark & 0.07 & 0.11 & 0.12 \\
\rowcolor{methodbg} \initname & 70B & \checkmark & 0.19 & 0.27 & 0.29 \\
\bottomrule
\end{tabular}
\caption{Efficiency comparison on TOFU benchmark. Time is measured in GPU hours (lower is better). Randomized SVD enables efficient initialization even for 70B-scale models.}
\label{tab:efficiency}
\end{table*}

\subsection{Efficiency Analysis}
Table~\ref{tab:efficiency} compares the initialization cost on the TOFU benchmark. On Llama2-7B, \initname demonstrates significantly faster initialization compared to FILA across all settings. Compared to VILA, \initname shows comparable efficiency on smaller forget sets, but becomes faster on the Forget 10\% setting (0.51 vs.\ 0.81 GPU hours).
This scalability advantage stems from a fundamental difference in computational requirements: while FILA and VILA need to compute gradients to estimate parameter importance maps, \initname only requires a forward pass over the forget set to collect layer outputs, followed by covariance computation and eigenvalue decomposition. The most time-consuming operation is the eigenvalue decomposition, whose cost remains nearly constant regardless of dataset size. As a result, \initname's efficiency advantage becomes increasingly pronounced with larger forget sets.

However, the eigenvalue decomposition cost depends on the layer dimension $d$, which grows with model size. This could become a bottleneck for larger models. Since \initname requires only the top-$r$ eigenvectors, we can leverage randomized SVD to efficiently approximate the leading eigenvectors. Standard eigenvalue decomposition has complexity $O(\min(d^2 n, d n^2))$ for an $n \times d$ activation matrix, whereas randomized SVD reduces this to $O(ndr + r^2(n+d))$, where $r \ll \min(n,d)$ is the target rank. For typical LoRA configurations where $r \le 64$ and $d \sim 4096$, this represents a substantial speedup. As shown in Table~\ref{tab:efficiency}, with randomized SVD, initialization on Llama2-7B drops to 0.07--0.12 GPU hours. Even on Llama2-70B, \initname maintains efficient initialization at 0.19--0.29 GPU hours, confirming its scalability to larger models.

\section{Conclusion}
In this work, we introduced \mname, a novel LoRA-based method for LLM unlearning. By shifting the focus from parameter importance to representation subspaces, \mname effectively addresses the challenges posed by the superposition phenomenon and parameter polysemanticity. Our approach leverages a balanced subspace initialization to align unlearning updates with forget-specific directions and an orthogonal regularization term to protect the principal directions of the retain set. Extensive experiments on the TOFU and WMDP benchmarks demonstrate that \mname consistently outperforms state-of-the-art baselines, achieving superior unlearning quality while maintaining high model utility. Our analysis further confirms that \mname successfully disentangles forget and retain representations, providing a robust and precise solution for selective forgetting in large language models.

\section*{Limitations}
Despite its effectiveness, \mname has several limitations. First, the method requires computing covariance matrices and performing eigenvalue decomposition for each layer, which introduces a one-time computational overhead during initialization. While this cost is significantly lower than full fine-tuning, it may become non-trivial for extremely large models or high-dimensional representations. Second, the quality of the identified subspaces depends on the representativeness of the small subsets of forget and retain data used for covariance estimation. If these subsets do not accurately capture the underlying distributions, the initialization and regularization may be less effective.
Third, our evaluation is primarily focused on the TOFU and WMDP benchmarks. While these are standard in the field, further investigation is needed to assess the generalizability of \mname across a broader range of domains and unlearning tasks.
Finally, the performance of \mname is sensitive to hyperparameters such as the LoRA rank $r$ and the regularization strength $\lambda$, which may require careful tuning for different model architectures and datasets.

\bibliography{custom}

\appendix

\section{Proof of Theorem~\ref{theorem:init-of-BA}}
\label{app:proof}

The proof of Theorem~\ref{theorem:init-of-BA} follows the mathematical framework established in SC-LoRA \citep{luo2025sclora}. We adapt their Theorem 1 to the unlearning problem by reinterpreting the positive task as the forget set and the negative task as the retain set. The key insight---that eigenvectors of the weighted covariance difference capture discriminative directions---applies naturally to our objective of maximizing impact on $\mathcal{D}_f$ while minimizing interference with $\mathcal{D}_r$.

To prove this theorem, we first introduce the concept of orthogonal projection operators.
\begin{definition}
    Suppose $S$ is a subspace of $\mathbb{R}^n$ of dimension $r$, and let $\{q_i\}_{i\in[r]}$ be an orthonormal basis of $S$,
    then the orthogonal projection operator onto $S$, denoted $\Pi_S$, is defined as:
    \begin{equation}
      \Pi_S(x) = \sum_{i=1}^r (q_i^\top x) q_i = \sum_{i=1}^r (q_iq_i^\top) x.
    \end{equation}
    Note: the selection of the orthonormal basis does not affect $\Pi_S$. 
\end{definition}

\begin{proof}
\label{proof:init-of-BA}
We prove this theorem in three steps: (1) establish the relationship between $\|\Delta W x\|_2^2$ and projection onto subspace $S$; (2) derive the expected projection energy in terms of covariance matrices; and (3) apply Ky Fan's theorem to show that eigenvectors of $\Cov_{\Delta}$ maximize the objective.

\paragraph{Step 1: Projection operator representation.}
Let $\Delta W = BA$. At initialization, $B_{\text{init}} = Q_r$ and $A_{\text{init}} = Q_r^\top W_0$, so $\Delta W = Q_r Q_r^\top W_0$. For any input $x$, let $h = W_0 x$ be the corresponding output representation. Since $Q_r$ is an orthonormal basis for the subspace $S$, we have $Q_r Q_r^\top = \Pi_S$. Thus:
\begin{equation*}
    \Delta W x = Q_r Q_r^\top W_0 x = \Pi_S(h).
\end{equation*}
Therefore, $\|\Delta W x\|_2^2 = \|\Pi_S(h)\|_2^2$.

\paragraph{Step 2: Expected projection energy.}
Let $\{v_i\}_{i\in[r]}$ be any orthonormal basis that spans $S$, and denote $\tilde{I}_r = \sum_{i=1}^r v_i v_i^\top$. From the orthonormality of $\{v_i\}_{i\in[r]}$, we have:
\begin{equation*}
  \begin{aligned}
    \tilde{I}_r^\top \tilde{I}_r &= \sum_{i=1}^{r} \sum_{j=1}^{r} v_i v_i^\top v_j v_j^\top \\
    &= \sum_{i=1}^{r} \sum_{j=1}^{r} v_i \langle v_i, v_j\rangle v_j^\top \\
    &= \sum_{i=1}^{r} \sum_{j=1}^{r} \delta_{ij} v_i v_j^\top \\
    &= \sum_{i=1}^{r} v_i v_i^\top = \tilde{I}_r.
  \end{aligned}
\end{equation*}

For either distribution $\mathcal{P} \in \{\mathcal{P}_F, \mathcal{P}_R\}$ with covariance $\Cov$, we have:
\begin{equation*}
  \begin{aligned}
    \mathbb{E}_{h \sim \mathcal{P}}\left[\|\Pi_S(h)\|_2^2\right]
    &= \mathbb{E}_{h \sim \mathcal{P}}\left[\|\tilde{I}_r h\|_2^2\right] \\
    &= \mathbb{E}_{h \sim \mathcal{P}}\left[ \mathrm{tr}\left(h^\top \tilde{I}_r^\top \tilde{I}_r h \right)\right] \\
    &= \mathbb{E}_{h \sim \mathcal{P}}\left[ \mathrm{tr}\left(h^\top \tilde{I}_r h \right)\right] \\
    &= \mathbb{E}_{h \sim \mathcal{P}}\left[ \mathrm{tr}\left(\tilde{I}_r h h^\top\right)\right] \\
    &= \mathrm{tr}\left(\tilde{I}_r \mathbb{E}_{h \sim \mathcal{P}} \left[h h^\top\right]\right) \\
    &= \mathrm{tr}\left(\tilde{I}_r \Cov \right).
  \end{aligned}
\end{equation*}

Substituting into Eq.~\ref{eq:norm_objective}, the reward function becomes:
\begin{equation*}
  \begin{aligned}
  R(S) &= (1-\beta)\mathbb{E}_{h \sim \mathcal{P}_{F}} \left[ \| \Pi_S(h) \|_2^2 \right] \\ &- \beta \mathbb{E}_{h \sim \mathcal{P}_{R}} \left[ \|\Pi_S(h)\|_2^2 \right] \\
  &= (1-\beta)\mathrm{tr}\left(\tilde{I}_r \Cov_{F} \right) - \beta  \mathrm{tr}\left( \tilde{I}_r \Cov_{R} \right) \\
  &= \mathrm{tr}\left( \tilde{I}_r \Cov_{\Delta} \right).
  \end{aligned}
\end{equation*}

\paragraph{Step 3: Optimality via spectral decomposition.}
Suppose the spectral decomposition of $\Cov_{\Delta}$ is $Q \Sigma Q^\top$, where $Q = (q_1, q_2, \dots, q_{d_{out}})$ is orthogonal and $\Sigma$ is diagonal with eigenvalues in descending order. Then:
\begin{equation*}
  \begin{aligned}
  R(S) &= \mathrm{tr}\left( \tilde{I}_r Q \Sigma Q^\top \right) \\
  &= \sum_{i=1}^{r}\mathrm{tr}\left( v_i v_i^\top Q \Sigma Q^\top \right) \\
  &= \sum_{i=1}^{r} v_i^\top Q \Sigma Q^\top v_i.
  \end{aligned}
\end{equation*}

Extend $\{v_i\}_{i\in[r]}$ to a complete orthonormal basis $\{v_i\}_{i=1}^{d_{out}}$ for $\mathbb{R}^{d_{out}}$, and denote $u_i = Q^\top v_i$. Since $Q$ is orthogonal, $\{u_i\}_{i=1}^{d_{out}}$ is also orthonormal. By Ky Fan's theorem, 
\begin{equation*}
\max_{\{v_i\}_{i\in[r]}} \sum_{i=1}^{r} v_i^\top Q \Sigma Q^\top v_i = \sum_{i=1}^{r} \Sigma_{ii},
\end{equation*}
and this maximum is achieved when $S = \mathrm{span}(\{q_1, q_2, \dots, q_r\})$, where $q_i$ are the top-$r$ eigenvectors of $\Cov_{\Delta}$. Therefore, initializing $B = Q_r$ and $A = Q_r^\top W_0$ maximizes the objective in Eq.~\ref{eq:norm_objective}.
\paragraph{Uniqueness under eigenvalue gap.} If the eigenvalues of $\Sigma$ satisfy $\lambda_r > \lambda_{r+1}$ (a strict gap), then the maximizing subspace is unique and equals $\mathrm{span}(\{q_1, \dots, q_r\})$. Let $V = (v_1, \dots, v_r)$ collect an orthonormal basis of $S$ and define $U = Q^\top V$. Using $Q$'s orthogonality,
\begin{equation*}
  R(S) = \sum_{i=1}^{r} v_i^\top Q \Sigma Q^\top v_i = \sum_{j=1}^{d_{out}} \Sigma_{jj} \sum_{i=1}^{r} U_{ji}^2.
\end{equation*}
By orthogonality, $0 \le \sum_{i=1}^{r} U_{ji}^2 \le 1$ and $\sum_{j=1}^{d_{out}} \sum_{i=1}^{r} U_{ji}^2 = r$. With a strict spectral gap, the maximum is attained if and only if
\begin{equation*}
  \sum_{i=1}^{r} U_{ji}^2 = \begin{cases}1, & 1 \le j \le r, \\ 0, & r+1 \le j \le d_{out},\end{cases}
\end{equation*}
which is equivalent to $\sum_{i=1}^{r} v_i v_i^\top = \sum_{i=1}^{r} q_i q_i^\top$, hence $S = \mathrm{span}(\{q_1, \dots, q_r\})$. When $\lambda_r = \lambda_{r+1}$ (no gap), any $r$-dimensional subspace within the top-eigenspace achieves the same maximum, matching SC-LoRA's discussion.
\end{proof}

\section{Concentration Bound for Empirical Covariance Estimation}
\label{app:cov-estimation}

This section provides a concentration bound for empirical covariance estimation, bridging the gap between the population-level covariance matrices used in Theorem~\ref{theorem:init-of-BA} and the empirical estimators used in practice. Under a mild bounded-activation assumption, the empirical covariance matrices concentrate sharply in spectral norm, and the balanced covariance used by \initname remains a stable approximation to its population counterpart.

\begin{theorem}[Spectral concentration of empirical covariance]
\label{thm:cov-concentration}
Let $h \in \mathbb{R}^{d}$ be a random representation vector satisfying $\|h\|_2 \le M$ almost surely, and define its covariance matrix
\begin{equation*}
\Sigma = \mathbb{E}[h h^\top].
\end{equation*}
Given i.i.d. samples $\{h_i\}_{i=1}^{N}$, let
\begin{equation*}
\hat{\Sigma} = \frac{1}{N}\sum_{i=1}^{N} h_i h_i^\top.
\end{equation*}
Then for any $\delta \in (0,1)$, with probability at least $1-\delta$,
\begin{equation*}
\|\hat{\Sigma} - \Sigma\|_2
\le
\frac{4M^2\log(2d/\delta)}{3N}
+
M^2\sqrt{\frac{2\log(2d/\delta)}{N}}.
\end{equation*}
In particular, when $N \ge \log(2d/\delta)$, there exists an absolute constant $C>0$ such that
\begin{equation*}
\|\hat{\Sigma} - \Sigma\|_2
\le
CM^2\sqrt{\frac{\log(2d/\delta)}{N}}.
\end{equation*}
\end{theorem}

\begin{proof}
Define
\begin{equation*}
X_i = \frac{1}{N}(h_i h_i^\top - \Sigma).
\end{equation*}
Then $\{X_i\}_{i=1}^{N}$ are independent, zero-mean, symmetric random matrices and
\begin{equation*}
\hat{\Sigma} - \Sigma = \sum_{i=1}^{N} X_i.
\end{equation*}

We bound the two quantities required by the matrix Bernstein inequality \citep[Theorem 6.1]{tropp2012userfriendly}.

\paragraph{Spectral norm bound.}
Since $\|h_i h_i^\top\|_2 = \|h_i\|_2^2 \le M^2$ and $\|\Sigma\|_2 \le \mathbb{E}\|h\|_2^2 \le M^2$, we have
\begin{equation*}
\|X_i\|_2
\le
\frac{\|h_i h_i^\top\|_2 + \|\Sigma\|_2}{N}
\le
\frac{2M^2}{N}.
\end{equation*}
Hence the Bernstein radius is $R = 2M^2/N$.

\paragraph{Variance bound.}
Using $\mathbb{E}[h h^\top] = \Sigma$ and $\|h\|_2^2 \le M^2$, we obtain
\begin{equation*}
\mathbb{E}[(h h^\top - \Sigma)^2]
=
\mathbb{E}[\|h\|_2^2 h h^\top] - \Sigma^2
\preceq
M^2 \Sigma.
\end{equation*}
Therefore,
\begin{equation*}
\begin{aligned}
\mathbb{E}[X_i^2]
&\preceq
\frac{M^2 \Sigma}{N^2}, \\
\sigma^2
:=
\left\|\sum_{i=1}^{N}\mathbb{E}[X_i^2]\right\|_2
&\le
\frac{M^2\|\Sigma\|_2}{N}
\le
\frac{M^4}{N}.
\end{aligned}
\end{equation*}

Applying matrix Bernstein inequality gives, for any $t>0$,
\begin{equation*}
\Pr\!\left[\left\|\sum_{i=1}^{N} X_i\right\|_2 \ge t\right]
\le
2d \exp\!\left(
-\frac{t^2/2}{\sigma^2 + Rt/3}
\right).
\end{equation*}
Set $L = \log(2d/\delta)$ and choose
\begin{equation*}
t
=
\frac{LR}{3} + \sqrt{\frac{L^2R^2}{9} + 2L\sigma^2}
\le
\frac{2LR}{3} + \sqrt{2L\sigma^2}.
\end{equation*}
Then the right-hand side is at most $\delta$, and substituting the bounds on $R$ and $\sigma^2$ yields
\begin{equation*}
\|\hat{\Sigma} - \Sigma\|_2
\le
\frac{4M^2\log(2d/\delta)}{3N}
+
M^2\sqrt{\frac{2\log(2d/\delta)}{N}}
\end{equation*}
with probability at least $1-\delta$. When $N \ge \log(2d/\delta)$, the $O(N^{-1})$ term is dominated by the $O(N^{-1/2})$ term up to a universal constant, giving the simplified bound.
\end{proof}

\begin{theorem}[Stability of the balanced covariance used by \initname]
\label{thm:delta-cov-stability}
Let $\hat{\Cov}_F$ and $\hat{\Cov}_R$ be empirical covariance estimators constructed from $N_f$ forget samples and $N_r$ retain samples, respectively, and define
\begin{equation*}
\hat{\Cov}_{\Delta} = (1-\beta)\hat{\Cov}_F - \beta \hat{\Cov}_R.
\end{equation*}
Assume the representations from both sets satisfy $\|h\|_2 \le M$ almost surely. Then with probability at least $1-\delta$,
\begin{equation*}
\|\hat{\Cov}_{\Delta} - \Cov_{\Delta}\|_2
\le
(1-\beta)\epsilon_f + \beta \epsilon_r,
\end{equation*}
where
\begin{equation*}
\begin{aligned}
\epsilon_f
&=
\frac{4M^2\log(4d/\delta)}{3N_f}
+
M^2\sqrt{\frac{2\log(4d/\delta)}{N_f}}, \\
\epsilon_r
&=
\frac{4M^2\log(4d/\delta)}{3N_r}
+
M^2\sqrt{\frac{2\log(4d/\delta)}{N_r}}.
\end{aligned}
\end{equation*}
Consequently, if the eigengap
\begin{equation*}
g = \lambda_r(\Cov_\Delta) - \lambda_{r+1}(\Cov_\Delta)
\end{equation*}
is strictly larger than $(1-\beta)\epsilon_f + \beta \epsilon_r$, then the top-$r$ eigenspace recovered from $\hat{\Cov}_{\Delta}$ is a stable perturbation of the population-optimal subspace, with principal-angle error controlled on the order of $\|\hat{\Cov}_{\Delta} - \Cov_{\Delta}\|_2 / g$ by standard eigenspace perturbation arguments.
\end{theorem}

\begin{proof}
Apply Theorem~\ref{thm:cov-concentration} to $\hat{\Cov}_F$ and $\hat{\Cov}_R$ separately with failure probability $\delta/2$ each, and take a union bound. On this event,
define $\Delta_F = \hat{\Cov}_F - \Cov_F$ and $\Delta_R = \hat{\Cov}_R - \Cov_R$. Then
\begin{equation*}
\begin{aligned}
\|\hat{\Cov}_{\Delta} - \Cov_{\Delta}\|_2
&=
\|(1-\beta)\Delta_F - \beta\Delta_R\|_2 \\
&\le
(1-\beta)\|\Delta_F\|_2
+
\beta\|\Delta_R\|_2 \\
&\le
(1-\beta)\epsilon_f + \beta \epsilon_r.
\end{aligned}
\end{equation*}
The eigenspace stability statement then follows from standard perturbation bounds for symmetric matrices: once the perturbation magnitude is smaller than the spectral gap $g$, the leading eigenspace of $\hat{\Cov}_{\Delta}$ remains close to that of $\Cov_{\Delta}$.
\end{proof}

\section{Algorithm}
\label{app:algorithm}

We provide the complete algorithmic description of \mname in Algorithm~\ref{alg:method}.

\begin{algorithm*}
\caption{Representation Subspace-Controlled Unlearning (\mname)}
\label{alg:method}
    \begin{algorithmic}[1]
        \Require Pre-trained model $\mathcal{M}$ with parameters $\boldsymbol{\theta}$, forget set $\mathcal{D}_f$, retain set $\mathcal{D}_r$, LoRA rank $r$, retain subspace dimension $k$, hyperparameters $\beta, \gamma, \lambda$
        
        \State \textbf{Phase 1: Initialization}
        \State Sample $N_f$ examples from $\mathcal{D}_f$ and $N_r$ examples from $\mathcal{D}_r$
        \For{all trainable parameters in $\mathcal{M}$}
            \State Feed forget samples, collect outputs $H_f = [h_1^{(f)}, h_2^{(f)}, \dots, h_{N_f}^{(f)}]^\top \in \mathbb{R}^{N_f \times d}$
            \State Feed retain samples, collect outputs $H_r = [h_1^{(r)}, h_2^{(r)}, \dots, h_{N_r}^{(r)}]^\top \in \mathbb{R}^{N_r \times d}$
            \State Compute $\Cov_F \leftarrow \frac{1}{N_f} H_f^\top H_f$
            \State Compute $\Cov_R \leftarrow \frac{1}{N_r} H_r^\top H_r$
            \State Compute $\Cov_\Delta \leftarrow (1-\beta) \Cov_F - \beta \Cov_R$
            \State Perform eigenvalue decomposition on $\Cov_\Delta$
            \State Extract top-$r$ eigenvectors $Q_r = (q_1, q_2, \dots, q_r)$
            \State Initialize $B_{\text{init}} \leftarrow Q_r$
            \State Initialize $A_{\text{init}} \leftarrow Q_r^\top W_0$
            \State Compute $W_{\text{res}} \leftarrow W_0 - B_{\text{init}} A_{\text{init}}$ \Comment{Residual weight; frozen during training}
            \State Perform eigenvalue decomposition on $\Cov_R$
            \State Extract top-$k$ eigenvectors $P_B = (p_1, p_2, \dots, p_k)$
        \EndFor
        
        \State \textbf{Phase 2: Training}
        \For{each training iteration}
            \State Sample mini-batch from $\mathcal{D}_f$ and $\mathcal{D}_r$
            \State Compute forget loss: $\mathcal{L}_{\text{forget}} \leftarrow \mathcal{L}_{\text{IHL}}(\mathcal{D}_f)$ or $\mathcal{L}_{\text{GA}}(\mathcal{D}_f)$
            \State Compute retain loss: $\mathcal{L}_{\text{retain}} \leftarrow \mathcal{L}_{\text{CE}}(\mathcal{D}_r)$
            \State Compute orthogonal loss: $\mathcal{L}_{\regname} \leftarrow \| B^\top P_B \|_F^2$
            \State Compute total loss: $\mathcal{L}_{\text{total}} \leftarrow \mathcal{L}_{\text{forget}} + \gamma \mathcal{L}_{\text{retain}} + \lambda \mathcal{L}_{\regname}$
            \State Update LoRA parameters $\{A, B\}$ via gradient descent
        \EndFor \\
        \Return Unlearned model $\mathcal{M}'$ with updated LoRA adapters
    \end{algorithmic}
\end{algorithm*}


\section{Implementation and Training Details}
\label{app:training-details}

\subsection{Hardware and Environment}
\label{app:hardware-env}
Unless otherwise specified, all experiments were conducted on a single NVIDIA L20 GPU (48GB VRAM).
Experiments on WMDP were conducted on a single NVIDIA A100 GPU (40GB VRAM).
Our implementation is based on PyTorch 2.5.1 (CUDA 12.1) and the Hugging Face ecosystem, including
Transformers 5.0.0.dev0, Tokenizers 0.22.1, and PEFT 0.17.1.
We used DeepSpeed with ZeRO Stage 2 for memory optimization and launched runs via \texttt{torchrun}.
All experiments were trained using BF16 precision by default.

\subsection{Models and LoRA Configurations}
\label{app:model-lora}
We report results on three backbone models depending on the benchmark: Llama2-7B and Phi-1.5B for TOFU,
and Zephyr-7B-$\beta$ for WMDP.
For parameter-efficient updates, we apply LoRA adapters to the following linear projections in every transformer block:
\{ \texttt{q\_proj}, \texttt{k\_proj}, \texttt{v\_proj}, \texttt{o\_proj}, \texttt{gate\_proj}, \texttt{up\_proj}, \texttt{down\_proj} \}.
Unless otherwise stated, we use LoRA rank $r=32$ and set the scaling factor to $\alpha = 2r$ (thus $\alpha=64$).
The LoRA dropout is set to 0.0 for TOFU and 0.05 for WMDP.

\subsection{Hyperparameters Search Space}
\label{app:hp-search}
We performed a grid sweep over key hyperparameters.
Across all experiments, we fix the retain loss weight $\gamma=1.0$ and the retain subspace dimension $k=128$.

\paragraph{TOFU.}
All TOFU runs were trained for 5 epochs with batch size 4 and gradient accumulation steps 8 (effective batch size 32),
using weight decay 0.01.
For Llama2-7B, we swept the learning rate over
$\{1\times10^{-5},\, 5\times10^{-5},\, 1\times10^{-4}\}$,
and swept $\lambda\in\{0.5,\,0.7\}$ and $\beta\in\{0.5,\,0.7\}$.
For Phi-1.5B, we used the same grid, additionally including $2\times10^{-4}$ in the learning-rate sweep,
i.e., $\{1\times10^{-5},\, 5\times10^{-5},\, 1\times10^{-4},\, 2\times10^{-4}\}$, and swept
$\lambda\in\{0.5,\,0.7\}$ and $\beta\in\{0.5,\,0.7\}$.

\paragraph{WMDP.}
We swept the learning rate over $\{1\times10^{-5},\, 3\times10^{-5},\, 5\times10^{-5}\}$ and
$\lambda\in\{0.1,\,0.5,\,0.7\}$, with \texttt{max\_steps}=100.
All WMDP experiments used LoRA rank $r=32$ and $\alpha=64$.

\subsection{Model Selection Criteria}
\label{app:model-selection}
For TOFU, we selected hyperparameters that yield a favorable trade-off between Forget Quality (FQ) and Model Utility (MU).
For WMDP, we enforced a utility constraint based on MMLU, requiring MMLU to be at least 95\% of the original model's performance.
Among feasible configurations, we selected those that minimize accuracy on the target corpora (bio/cyber) within the swept hyperparameter grid.

\subsection{Best Configurations Reported for WMDP}
\label{app:wmdp-best-config}
For reproducibility, we report the exact best-performing configurations used in our main WMDP results.
All configurations below are evaluated on checkpoints of steps=100.
\begin{itemize}
  \item \textbf{GD + \mname(bio):} $r=32$, $\alpha=64$, lr $=1\times10^{-5}$, $\beta=0.7$, $\lambda=0.5$.
  \item \textbf{GD + \mname(cyber):} $r=32$, $\alpha=64$, lr $=3\times10^{-5}$, $\beta=0.5$, $\lambda=0.5$.
  \item \textbf{IHL + \mname(bio):} $r=32$, $\alpha=64$, lr $=5\times10^{-5}$, $\beta=0.5$, $\lambda=0.1$.
  \item \textbf{IHL + \mname(cyber):} $r=32$, $\alpha=64$, lr $=3\times10^{-5}$, $\beta=0.5$, $\lambda=0.5$.
\end{itemize}

\end{document}